    \documentclass[11pt]{article}
    
    \usepackage[margin=1in]{geometry}
    \usepackage{setspace}
    \usepackage{microtype}
    \usepackage{graphicx}
    \usepackage{natbib}
    \usepackage{subcaption}
    \usepackage{booktabs}
    \usepackage{algorithm}
    \usepackage{algorithmic}
    \usepackage{xcolor}
    \usepackage{hyperref}
    \usepackage{soul}
    
    \usepackage{amsmath}
    \usepackage{amssymb}
    \usepackage{mathtools}
    \usepackage{amsthm}
    \usepackage[capitalize,noabbrev]{cleveref}
    
    \theoremstyle{plain}
    \newtheorem{theorem}{Theorem}[section]
    \newtheorem{proposition}[theorem]{Proposition}

    \theoremstyle{definition}
    \newtheorem{definition}[theorem]{Definition}
    
    \theoremstyle{remark}

    \renewcommand{\arraystretch}{1.2}
    
    \title{\textbf{Learning Stable Predictors from Weak Supervision under Distribution Shift}}
    
    \author{
    Mehrdad Shoeibi \\
    University of Central Florida \\
    \texttt{me604598@ucf.edu}
    \and
    Elias Hossain \\
    University of Central Florida
    \and
    Ivan Garibay \\
    University of Central Florida
    \and
    Niloofar Yousefi \\
    University of Central Florida \\
    \texttt{Niloofar.Yousefi@ucf.edu}
    }
    
    \date{}
    
\begin{document}
    
    \maketitle
    
    \begin{abstract}
    Learning from weak, proxy, or relative supervision is common when ground-truth labels are unavailable, but robustness under distribution shift remains poorly understood because the supervision mechanism itself may change across environments. We formalize this phenomenon as supervision drift, defined as changes in $P(y \mid x, c)$ across contexts, and study it in CRISPR-Cas13d transcriptomic perturbation experiments where guide efficacy is inferred indirectly from RNA-seq responses. Using publicly available data spanning two human cell lines and multiple post-induction timepoints, we construct a controlled non-IID benchmark with explicit domain (cell line) and temporal shifts, while reusing a fixed weak-label construction across all contexts to avoid changing targets. Across linear and tree-based models, weak supervision supports meaningful learning in-domain (ridge $R^2 = 0.356$, Spearman $\rho = 0.442$) and partial cross-cell-line transfer (ridge $\rho = 0.399$, random forest $\rho = 0.364$). In contrast, temporal transfer collapses across all model classes considered, yielding negative $R^2$ and weak or near-zero $\rho$ (ridge $R^2 = -0.145$, $\rho = 0.008$; XGBoost $R^2 = -0.155$, $\rho = 0.056$; random forest $R^2 = -0.322$, $\rho = 0.139$). Additional robustness analyses using externally recomputed weak labels, shift-score quantification, and simple mitigation baselines preserve the same qualitative pattern. Feature--label association and feature-importance analyses remain relatively stable across cell lines but change sharply over time, indicating that failures arise from supervision drift rather than model capacity or simple covariate shift. These results show that strong in-domain performance under weak supervision can be misleading and motivate feature stability as a lightweight diagnostic for non-transferability before deployment.
    \end{abstract}
    
    \section{Introduction}
    \label{sec:introduction}
    
    Machine learning is increasingly used in settings where ground-truth labels are unavailable, expensive, or ill-defined. In such cases, models are trained with weak, proxy, or relative supervision, namely signals derived indirectly from downstream measurements, comparisons, or aggregated observations. While weak supervision enables learning at scale, it also changes what it means to generalize. The observed labels may be noisy, biased, and only partially aligned with the underlying quantity of interest. As a result, strong in-distribution performance does not necessarily imply robustness under distribution shift.
    
    A central challenge in this setting is that distribution shifts may affect not only the input distribution $P(x)$, but also the relationship between inputs and supervision. In many real-world pipelines, the supervision signal is itself the output of a measurement and processing procedure that varies across experimental contexts, such as cell type, protocol, or timepoint. This introduces a failure mode not captured by classical covariate shift or label shift. The conditional relationship between features and weak labels may change across environments. We refer to this phenomenon as \emph{supervision drift}, defined by $P(y \mid x, c) \neq P(y \mid x, c')$ across contexts $c, c'$.
    
    Importantly, this work does not aim to introduce a new predictive model or optimization algorithm. Instead, it identifies and characterizes a previously underexamined failure mode in weakly supervised learning. Our contribution lies in formalizing supervision drift, designing a controlled evaluation protocol that isolates it, and demonstrating its practical consequences through a real-world biological case study.
    
    We study this phenomenon in transcriptomic perturbation experiments using CRISPR-Cas13d, where direct measurements of perturbation efficacy are unavailable and supervision is inferred indirectly from RNA-seq responses. Weak labels are constructed from relative transcriptomic responses using a fixed reference signal and reused across all contexts to avoid target redefinition. We then evaluate predictors under two structured shifts: domain shift across cell lines and temporal shift across post-induction timepoints, training in a single source context and evaluating without adaptation.
    
    Our results reveal a clear asymmetry. Within a fixed biological context, weak supervision supports meaningful learning. Ridge regression achieves nontrivial predictive accuracy and rank recovery in-domain ($R^2 = 0.356$, $\rho = 0.442$). Under cross-cell-line transfer, performance degrades but remains moderate for both linear and nonlinear models (ridge $\rho = 0.399$, random forest $\rho = 0.364$), which suggests that some determinants of the weak supervision signal are conserved across domains. In contrast, temporal transfer fails. Models trained at an early timepoint do not reliably predict later-stage signals, yielding negative $R^2$ and near-zero rank correlation. Increasing model capacity does not resolve this collapse. For example, XGBoost improves in-domain performance but still fails under temporal shift ($R^2 = -0.155$, $\rho = 0.056$).
    
    The significance of this problem extends beyond the specific biological setting studied here. Weak supervision is widely used in modern machine learning systems, including scientific pipelines, recommendation systems, and large-scale models trained using indirect or proxy signals. In such settings, the relationship between observed supervision and the underlying target may change across environments, even when the model and feature representation remain unchanged. Our goal is therefore to highlight a practical generalization risk that may be overlooked when evaluation focuses only on in-domain performance.
    
    While our empirical study is conducted in a specific transcriptomic perturbation setting, it should be interpreted as a controlled case study rather than a claim of universal generalization. Our aim is to show how supervision drift can arise in a realistic setting and how simple feature-stability signals can reveal non-transferability. Establishing the prevalence of this phenomenon across other weak supervision domains remains an important direction for future work.
    
    Our contributions are as follows:
    
    \begin{itemize}
        \item We formalize supervision drift as a distinct failure mode in weakly supervised learning under distribution shift.
        \item We propose a controlled non-IID evaluation protocol that isolates supervision drift by fixing the weak-label construction across contexts.
        \item We empirically demonstrate partial cross-domain robustness but severe temporal non-transferability across model classes.
        \item We show that feature-stability analysis provides a lightweight diagnostic for detecting non-transferability without access to ground truth.
        \item We strengthen the empirical analysis with leakage-robust label reconstruction, shift-score quantification, and simple mitigation baselines, all of which preserve the same qualitative conclusions.
        \item We position our findings as a controlled empirical case study that highlights a practical risk in weak supervision, rather than claiming broad generalization across all domains.
    \end{itemize}
    
    \section{Related Work}
    \label{sec:related_work}
    
    This work relates to three main areas: weak and relative supervision, generalization under distribution shift, and feature stability as a diagnostic signal.
    
    \paragraph{Weak and Relative Supervision.}
    Weak supervision addresses settings where labels are indirect, noisy, or derived from proxy signals \citep{ratner2019training,chen2024general}. Relative or comparison-based supervision has also been studied in ranking and preference-learning settings, where supervision encodes ordering rather than absolute values \citep{bao2020pairwise,dan2021learning}. However, most prior work assumes a fixed relationship between features and supervision signals, which limits applicability under changing experimental conditions.
    
    \paragraph{Generalization under Distribution Shift.}
    A large body of work studies robustness under distribution shift, including domain adaptation, domain generalization, and invariant learning. In particular, causality-inspired approaches such as invariant risk minimization (IRM) \citep{arjovsky2019invariant}, invariant prediction \citep{peters2016causal}, and invariant models for transfer learning \citep{rojas2018invariant} aim to identify features whose predictive relationships remain stable across environments. Distributionally robust optimization (DRO) and adversarial approaches similarly aim to improve worst-case generalization across domains. These methods typically assume access to environment labels and focus on learning invariant predictors, whereas our work focuses on diagnosing failures when supervision itself may drift.
    
    \paragraph{Feature Stability and Interpretability.}
    Feature stability has been studied as a property of reliable models and feature selection procedures \citep{nogueira2018stability}. Recent work connects the stability of explanations or feature importance to robustness under distribution shift, showing that unstable feature reliance often correlates with poor generalization. Our approach builds on this intuition but focuses specifically on weak supervision settings, where instability may arise from changes in the supervision mechanism rather than from model overfitting alone.
    
    \paragraph{Positioning of This Work.}
    Unlike prior work that assumes stable supervision or focuses on learning invariant predictors, this study examines a setting in which the supervision signal itself may change across contexts. We introduce supervision drift as a practical failure mode in weakly supervised learning and demonstrate, through a controlled transcriptomic case study, how feature stability can act as a lightweight diagnostic for non-transferability without requiring clean labels or retraining. We do not claim novelty in feature stability itself as a general machine learning concept. Rather, our contribution is to identify supervision drift as a practically observable failure mode in a weakly supervised biological setting and to show that simple stability statistics can serve as an interpretable diagnostic in that setting.
    
    \section{Preliminaries}
    \label{sec:preliminaries}
    
    \subsection{Weak and Relative Supervision}
    
    In many scientific machine learning settings, clean ground-truth labels are impractical or unavailable. Supervision is therefore often derived from indirect signals, such as proxy measurements or downstream responses. Such supervision is weak because it is noisy, biased, and only partially informative of the underlying quantity of interest.
    
    The focus is on settings where labels encode \emph{relative effects} inferred from indirect observations rather than absolute outcomes. The supervision signal provides comparative or ordinal information across targets under a given condition, without corresponding to true causal effects. Prediction tasks in this regime may be formulated as regression over relative scores, ranking, or coarse-grained classification, with the objective of recovering stable relative patterns despite indirect supervision.
    
    \subsection{Structured Distribution Shifts}
    
    Training and evaluation data often differ in distribution, which violates the IID assumption. We focus on \emph{structured distribution shifts}, where variation occurs along interpretable dimensions such as domain (e.g., cell line), time (e.g., post-induction stage), or repeated experimental measurements.
    
    Such shifts naturally arise in transcriptomic perturbation experiments, where perturbation efficacy is inferred from relative gene expression changes across guides, timepoints, and experimental conditions. The resulting data exhibit structured variation along biologically meaningful axes, which provides a realistic setting for studying robustness under weak supervision.
    
    \subsection{Supervision Drift}
    
    In weakly supervised learning, the relationship between input features and supervision signals is indirect and context dependent. Under structured distribution shifts, this relationship may itself change across environments.
    
    \begin{definition}[Supervision Drift]
    Supervision drift occurs between two contexts $c$ and $c'$ if the conditional relationship between input features $x$ and the weak supervision signal $y$ differs across contexts, that is,
    \[
    P(y \mid x, c) \neq P(y \mid x, c').
    \]
    \end{definition}
    
    Supervision drift captures changes in the supervision mechanism rather than shifts in the input distribution alone. Consequently, strong in-domain performance within a given context does not necessarily imply reliable generalization across contexts.
    
    \begin{proposition}[Non-transferability under Supervision Drift]
    Under supervision drift between a source context $c$ and a target context $c'$, strong in-domain performance within $c$ does not guarantee reliable generalization to $c'$, even when the feature representation and learning objective are fixed.
    \end{proposition}
    
    \paragraph{Relation to Concept Drift and Domain Generalization.}
    Supervision drift is related to, but distinct from, classical notions of concept drift and domain generalization. Concept drift typically refers to changes in the conditional distribution $P(y \mid x)$ over time, under the assumption that labels correspond to a stable ground-truth target. In contrast, in weakly supervised settings, the observed labels are themselves indirect measurements that depend on context. As a result, supervision drift reflects changes in the mapping from inputs to proxy supervision signals, rather than changes in an underlying ground-truth function alone. Similarly, while domain generalization methods assume multiple environments and aim to learn invariant predictors, our setting focuses on diagnosing when such invariance fails because the supervision mechanism is unstable.
    
    \section{Method}
    
    This study investigates learning from weak and relative supervision in transcriptomic perturbation data, with a focus on robustness under structured distribution shifts. Instead of predicting absolute perturbation efficacy, we study when predictive signals derived from indirect transcriptomic responses remain transferable across different biological contexts, such as across cell lines and post-induction timepoints. While the predictive task is intentionally simplified to isolate the effect of supervision drift, it reflects a practical setting in transcriptomic analysis where direct measurements of perturbation efficacy are unavailable and proxy signals derived from downstream molecular responses are commonly used.
    
    \subsection{Method Overview}
    
    \begin{figure*}[t]
        \centering
        \includegraphics[width=\textwidth]{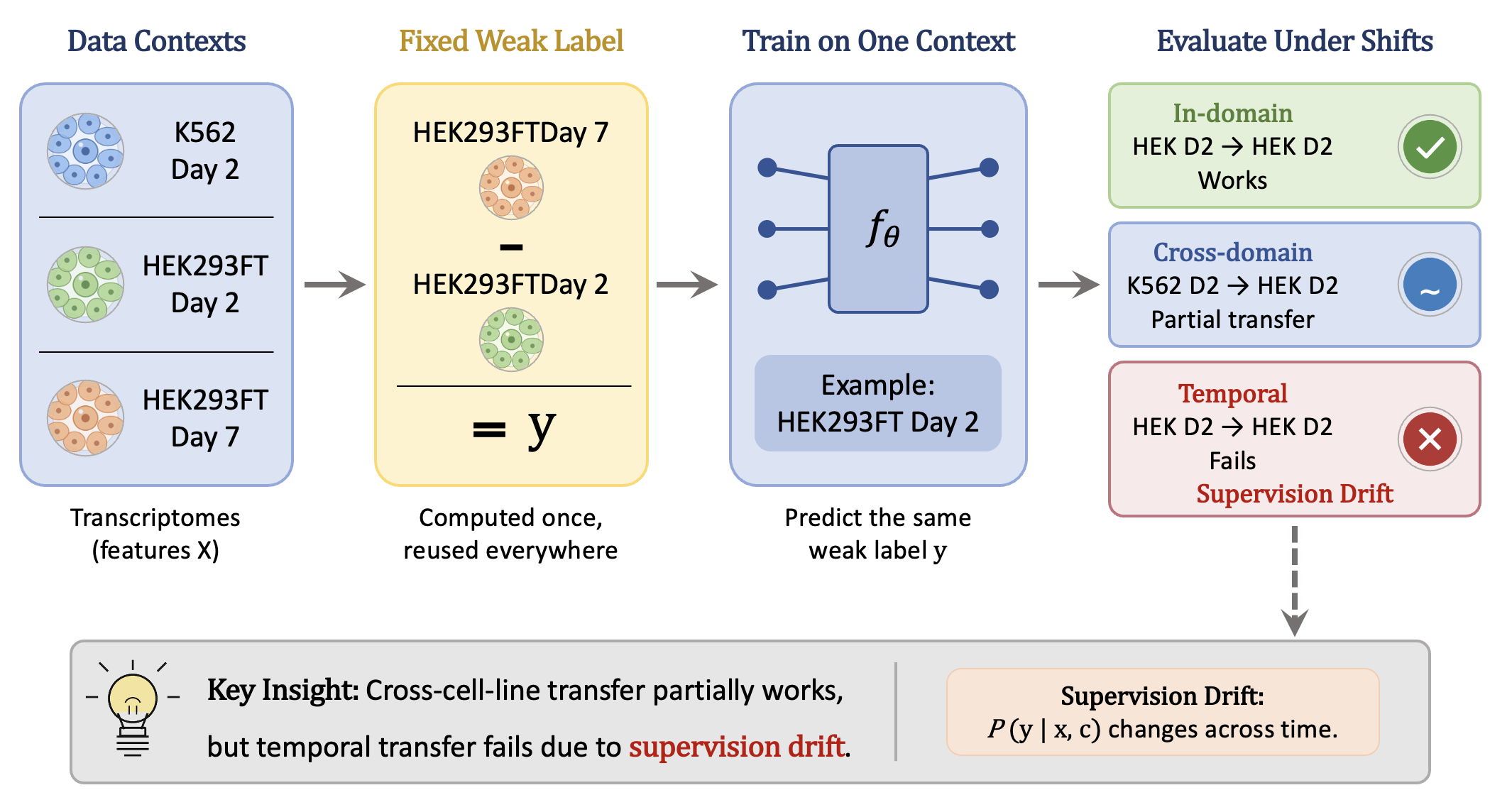}
     \caption{Overview of the experimental framework for studying generalization under weak supervision.
Transcript-level features are extracted from multiple biological contexts (K562 Day~2, HEK293FT Day~2, and HEK293FT Day~7).
A single fixed weak supervision signal $y$ is constructed once from the HEK293FT Day~7 minus Day~2 transcriptomic response and reused across all experiments, ensuring that the target remains unchanged across contexts. Models are trained within a single source context (e.g., HEK293FT Day~2 or K562 Day~2) to predict this fixed weak label, without access to other contexts during training.
Generalization is then evaluated under three settings: (i) in-domain (HEK293FT Day~2 $\rightarrow$ HEK293FT Day~2), (ii) cross-domain (K562 Day~2 $\rightarrow$ HEK293FT Day~2), and (iii) temporal (HEK293FT Day~2 $\rightarrow$ HEK293FT Day~7). By holding the supervision signal fixed and varying only the feature context, this design isolates changes in the feature--supervision relationship.
While partial transfer is observed across cell lines, temporal generalization fails, revealing supervision drift, where the conditional relationship $P(y \mid x, c)$ changes across contexts.}
        \label{fig:architecture_overview}
    \end{figure*}
    
    As illustrated in Figure~\ref{fig:architecture_overview}, we study a learning setting in which direct measurements of perturbation efficacy are unavailable and supervision is therefore derived indirectly from transcriptomic responses, yielding weak and relative labels defined at the transcript level. In this setting, only downstream molecular responses are observed, while the true causal efficacy of perturbations remains unobserved.
    
    Models are trained in a single source context and evaluated under explicitly defined structured distribution shifts, without any test-time adaptation. We consider two shift dimensions: (i) cross-cell-line transfer, where models are trained on \textbf{K562 Day~2} and evaluated on \textbf{HEK293FT Day~2}, and (ii) temporal transfer, where models are trained on \textbf{HEK293FT Day~2} and evaluated on \textbf{HEK293FT Day~7}. This design enables a controlled comparison of generalization behavior across domain and temporal shifts while holding the supervision mechanism fixed, which allows us to separate failures due to distribution shift from failures due to context-dependent changes in the supervision signal itself.
    
    To isolate the effects of weak supervision and distribution shift, our approach relies exclusively on transcript-level expression features and deliberately excludes guide sequence, RNA structural information, or contextual covariates from the input space. Weak supervision signals are computed once using a fixed reference aggregation and reused across all contexts, ensuring that any observed performance changes arise from shifts in the feature distribution or in the feature--label relationship.
    
    Beyond predictive accuracy, we assess robustness using post hoc diagnostics that examine feature--label associations and feature-importance stability across contexts. These analyses help distinguish transfer driven by invariant transcriptomic signals from transfer driven by context-specific correlations, and they support identification of failures attributable to supervision drift rather than model capacity limitations.
    
    We treat each transcript as a prediction instance, using transcript-level expression features as input and a fixed transcript-level weak response score as the prediction target (see Appendix~\ref{app:unit_def}).
    
    \subsection{Data Splits and Label Usage}
    \label{sec:data_splits_label_usage}
    
    A central design choice in this study is that the weak supervision target is fixed once and then reused across all experimental settings. This avoids target redefinition across contexts and makes it possible to study whether the relationship between features and supervision remains stable under domain and temporal shift.
    
    \paragraph{Global supervision signal.}
    For each aligned transcript $i$, the target $y_i$ is a fixed weak response score derived once from the HEK293FT Day~7 versus Day~2 contrast and then reused across all contexts. Importantly, this supervision vector is not recomputed separately for K562, HEK293FT Day~2, or HEK293FT Day~7. All models in all settings therefore predict the same transcript-level target vector.
    
    \paragraph{Feature contexts.}
    Each context provides a feature matrix whose rows are aligned by transcript identity. In this study, the contexts are K562 Day~2, HEK293FT Day~2, and HEK293FT Day~7. Because the target is fixed, performance differences across settings reflect changes in the relationship between context-specific features and the same weak supervision signal, rather than changes in the target definition itself.
    
    \paragraph{Experimental settings.}
    Table~\ref{tab:data_splits} summarizes the exact feature--label assignments used in the three evaluation settings.
    
    \begin{table}[t]
    \centering
    \caption{Data assignments for the three evaluation settings. The weak supervision target $\mathbf{y}$ is the same fixed transcript-level vector in all cases.}
    \label{tab:data_splits}
    \begin{tabular}{llll}
    \toprule
    Setting & Training features & Test features & Target \\
    \midrule
    In-domain & HEK293FT Day~2 & HEK293FT Day~2 & fixed $\mathbf{y}$ \\
    Cross-domain & K562 Day~2 & HEK293FT Day~2 & fixed $\mathbf{y}$ \\
    Temporal & HEK293FT Day~2 & HEK293FT Day~7 & fixed $\mathbf{y}$ \\
    \bottomrule
    \end{tabular}
    \end{table}
    
    For the in-domain setting, an 80/20 split is used within HEK293FT Day~2 to assess standard predictive performance under matched train and test conditions. For the cross-domain setting, the model is trained using K562 Day~2 features and evaluated using HEK293FT Day~2 features, while the target vector remains the same. This does not assume that one cell line predicts the biology of another in a causal sense. Rather, it tests whether the feature--supervision relationship learned in one cellular context transfers to another under a fixed supervision signal. For the temporal setting, the model is trained using HEK293FT Day~2 features and evaluated using HEK293FT Day~7 features, again with the same target vector. In this setting, failure of transfer indicates that the alignment between features and the fixed supervision signal is not preserved over time.
    
    \paragraph{Leakage clarification.}
    Using a target derived from the HEK293FT Day~7 versus Day~2 contrast does not imply that Day~7 test features are available during training. The model never sees HEK293FT Day~7 features in the temporal training stage. It only receives the fixed target vector and Day~2 feature inputs. In this sense, the study does not involve ordinary test-set leakage. Rather, it asks whether early-timepoint features can predict a fixed downstream weak supervision signal. If substantial leakage were present, one would expect inflated temporal performance. Instead, temporal transfer collapses, which is inconsistent with a simple leakage explanation.
    
    \subsection{Modeling Components}
    
    \paragraph{Feature Representation.}
    Each transcript is represented using Salmon abundance estimates, including log-transformed transcript-per-million values (logTPM) and within-sample rank percentiles. These features capture both absolute expression magnitude and relative ordering. Experimental context variables are used only to define train--test splits and are not provided as model inputs.
    
    A potential concern is that rank-based features may implicitly encode information related to the relative comparisons used to construct the weak labels, thereby introducing a form of structural leakage. To address this, we explicitly evaluate models using (i) absolute expression features only (logTPM), (ii) rank-based features only, and (iii) their combination. As shown in the ablation results, models trained on absolute expression alone exhibit similar qualitative behavior, including failure under temporal shift. This suggests that the observed temporal non-transferability is unlikely to be explained solely by leakage from rank-based features.
    
    To improve readability outside the transcriptomics literature, we use a few domain-specific terms in a narrow technical sense. Here, a \emph{guide} refers to the targeting molecule used in a CRISPR perturbation experiment, \emph{Salmon abundance estimates} refer to transcript quantification values produced by a standard RNA-seq quantification tool, and \emph{transcript-level expression features} simply denote numeric features derived from measured transcript abundance. Likewise, the term \emph{fixed reference signal} refers to a weak-label construction that is computed once and then reused across all evaluation contexts.
    
    \paragraph{Weak and Relative Supervision.}
    True perturbation efficacy is unobserved. Supervision is derived from relative transcriptomic responses. Weak labels are defined using a fixed reference signal computed as the HEK293FT Day~7 minus Day~2 response. Labels are treated as context invariant and are reused across all training and evaluation contexts to prevent target redefinition.
    
    \paragraph{Weak-label construction and leakage control.}
    Weak supervision labels are constructed once using the HEK293FT Day~7 minus Day~2 transcriptomic response, prior to any train--test splitting. This label construction does not depend on features from the training or evaluation contexts, and no information from test contexts is used to recompute or adapt weak labels at evaluation time. The same fixed weak-label vector is reused across all in-domain, cross-domain, and temporal evaluations, ensuring that observed performance changes arise from differences in feature context rather than moving targets. The reference signal therefore serves as a fixed supervision anchor throughout the study.
    
    \paragraph{Predictive Models and Learning Objective.}
    Given a dataset $\mathcal{D} = \{(x_i, c_i, y_i)\}_{i=1}^N$, a predictor
    \[
    f_\theta : \mathbb{R}^d \rightarrow \mathbb{R}
    \]
    is trained within a single context by minimizing empirical risk using standard regression losses. Models are evaluated in-domain and under distribution shifts. Hypothesis classes with increasing capacity are considered, including ridge regression and tree-based ensembles, while holding features and supervision fixed. Evaluation emphasizes both regression metrics and rank-based metrics, reflecting the relative nature of supervision. Model hyperparameters and the fixed-hyperparameter evaluation protocol are provided in Appendix~\ref{app:hyperparams}.
    
    Algorithm~\ref{alg:weak_learning} summarizes the learning and evaluation pipeline.
    
    \subsection{Feature Stability Analysis}
    
    Robustness beyond predictive accuracy is assessed through feature importance stability across contexts. For each context $c$, a feature importance vector $\boldsymbol{\phi}^{(c)}$ is computed using coefficient magnitudes for linear models and impurity-based measures for tree-based models. Stability between contexts is quantified via Spearman rank correlation between importance rankings. Feature ablation experiments using expression-only, rank-only, and combined representations are conducted to further diagnose context-specific reliance.
    
    \subsection{Theoretical Analysis}
    \label{sec:theoretical_analysis}
    
    This section provides theoretical justification for learning under weak and relative supervision under structured distribution shifts. We formalize conditions under which relative supervision yields identifiable predictive rankings and characterize when such rankings remain transferable across contexts. The purpose of including this theoretical section is not to provide a complete or rigorous characterization of supervision drift, but to formalize the intuition behind the empirical observations and clarify under what simplified conditions feature stability can be expected to relate to transferability. Given that the main contribution of this work is empirical, the theoretical component is intended as a conceptual complement rather than a standalone contribution.
    
    \paragraph{Theorem 1 (Identifiability under Relative Supervision).}
    Assume that each transcript has an unobserved true perturbation efficacy and that weak labels are generated by a strictly monotonic transformation corrupted by bounded noise. If transcript-level features preserve relative ordering within a context, then empirical risk minimization on weak labels yields a ranking-consistent predictor.
    
    \paragraph{Theorem 2 (Feature Stability and Transferability).}
    Let $\boldsymbol{\phi}^{(c)}$ denote the feature importance ranking in context $c$. If feature importance rankings remain stable across contexts, then degradation in rank-based predictive performance under context shift is bounded as a function of this stability. Formal assumptions and proofs are provided in Appendix~\ref{app:theory}.
    
    These results are intended as intuition-building guarantees under simplifying assumptions rather than claims of tight optimality.
    
    \begin{algorithm}[t]
    \caption{Learning under Weak Supervision with Structured Distribution Shifts}
    \label{alg:weak_learning}
    \begin{algorithmic}[1]
    \REQUIRE Dataset $\mathcal{D} = \{(x_i, c_i, y_i)\}_{i=1}^N$, model class $\mathcal{F}$, training context $c_{\text{train}}$, test contexts $\mathcal{C}_{\text{test}}$
    \STATE Extract training set $\mathcal{D}_{c_{\text{train}}}$
    \STATE Train model $f_\theta \in \mathcal{F}$ by minimizing empirical risk on $\mathcal{D}_{c_{\text{train}}}$
    \FOR{each test context $c \in \mathcal{C}_{\text{test}}$}
        \STATE Evaluate $f_\theta$ on $\mathcal{D}_c$ using regression and rank-based metrics
        \STATE Compute context-specific feature importance $\boldsymbol{\phi}^{(c)}$
    \ENDFOR
    \STATE Output predictive performance and feature stability results
    \end{algorithmic}
    \end{algorithm}
    
    In practice, feature stability is computed as follows. For each context, a feature importance vector is extracted from the trained model using coefficient magnitudes for linear models or impurity-based importance for tree-based models. These importance values are converted into ranked feature lists, and stability between contexts is quantified using Spearman rank correlation between these rankings. The resulting correlation provides a scalar measure of how consistently the model relies on the same features across contexts.
    
    These theoretical results are intended to provide intuition for why feature stability may be associated with transferability under weak supervision, rather than to establish tight guarantees. The assumptions are simplified and may not fully capture the complexity of real transcriptomic systems. Accordingly, the theory should be interpreted as a conceptual justification that complements the empirical findings, rather than as a formal characterization of optimality.
    
    \section{Experimental Setup}
    \label{sec:experimental_setup}
    
    \paragraph{Data overview.}
    Public CRISPR-Cas13d transcriptomic perturbation datasets generated via bulk RNA-seq in human cell lines are used \cite{SRX30894812_2025, SRX30894811_2025, SRX30894813_2025}. Perturbation efficacy is unobserved and inferred from transcript-level expression changes, yielding weak and relative supervision. HEK293FT measurements are available at Day~2 and Day~7, which enables temporal generalization analysis. HEK293FT Day~2 serves as the primary training and in-domain evaluation context, while cross-domain robustness is assessed by training on K562 Day~2 and evaluating on HEK293FT Day~2. Each sample corresponds to a transcriptome-wide expression profile at the transcript level. This setting is scientifically relevant because transcriptomic perturbation studies frequently rely on indirect molecular readouts rather than direct measurements of perturbation efficacy, which makes weak supervision a practical necessity rather than an artificial design choice.
    
    \paragraph{RNA-seq preprocessing.}
    Raw FASTQ files were processed using a standard RNA-seq pipeline. Quality control was performed with FastQC and MultiQC, followed by transcript-level quantification using Salmon with an hg38 reference. Samples with near-zero mapping rates were removed. No batch correction or cross-context normalization was applied, because differences across contexts represent meaningful distribution shifts. Additional preprocessing details are provided in the appendix.
    
    \paragraph{Evaluation protocol.}
    Experimental contexts are defined by cell line and post-induction timepoint, which induces structured domain and temporal shifts. Models are trained in a single source context and evaluated without adaptation across predefined target contexts: in-domain (HEK293FT Day~2), cross-domain (K562 Day~2 $\rightarrow$ HEK293FT Day~2), and temporal (HEK293FT Day~2 $\rightarrow$ Day~7). Performance is primarily assessed using rank-based metrics, supplemented with feature--label correlation and feature stability analyses across contexts.
    
    \section{Results}
    \label{sec:results}
    
    \subsection{In-Domain Predictability}
    
    Figure~\ref{fig:in_domain} summarizes in-domain predictive performance under weak supervision, where models are both trained and evaluated within the same biological context (HEK293FT, Day~2). Results are reported using the combined transcript-level feature representation (logTPM + rank). Despite indirect and noisy supervision, both linear and nonlinear models achieve nontrivial predictive accuracy. Ridge regression attains $R^2 = 0.356$ and Spearman rank correlation $\rho = 0.442$, while a depth-capped random forest ($\mathtt{max\_depth}=10$, 200 trees) attains $R^2 = 0.374$ and $\rho = 0.383$. The two model families thus recover comparable relative structure in-domain, with random forest achieving a slightly higher $R^2$ and ridge regression a slightly higher rank correlation. These results indicate that transcriptomic features capture meaningful relative signals associated with Cas13d-mediated knockdown and that stable relative structure is recoverable by inductive biases of different capacities. Overall, relative perturbation effects are partially identifiable within a fixed biological context, even in the absence of direct efficacy labels.
    
    \begin{figure}[t]
        \centering
        \includegraphics[width=\columnwidth]{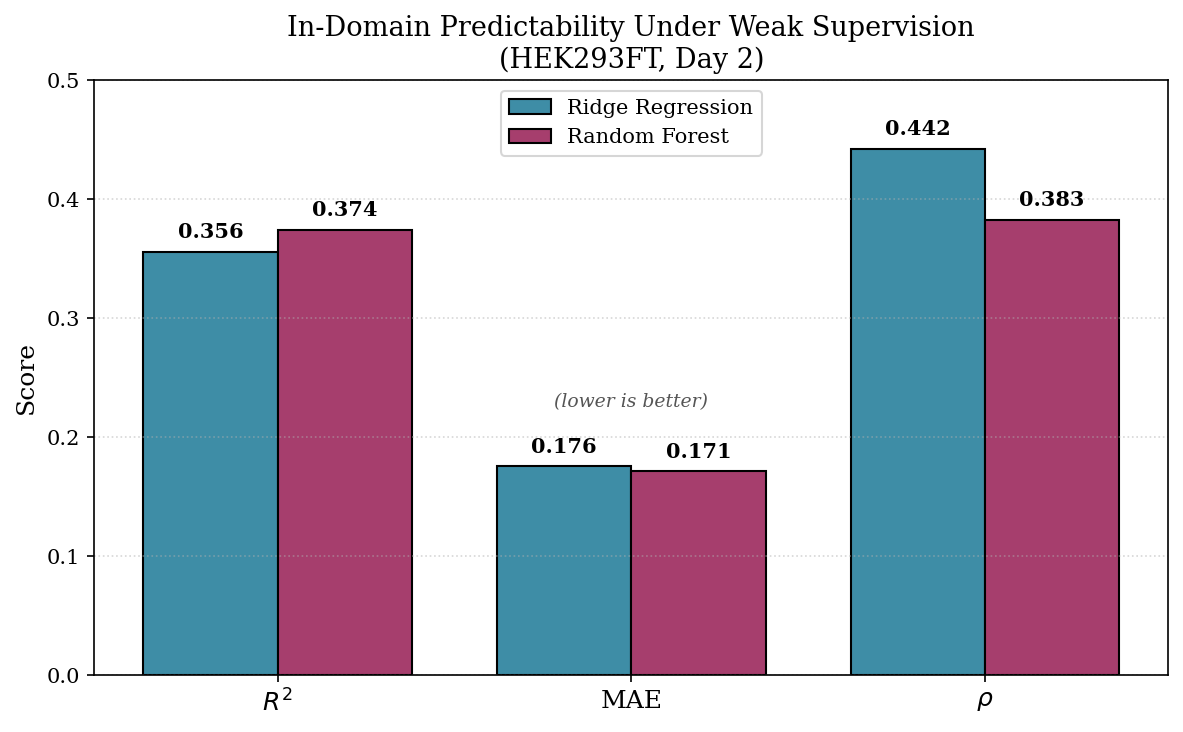}
        \caption{In-domain predictability under weak supervision (HEK293FT, Day~2). Performance is reported using $R^2$, mean absolute error (MAE; lower is better), and Spearman rank correlation $\rho$. Ridge regression and a depth-capped random forest achieve comparable in-domain performance, indicating that transcript-level features contain recoverable relative signal under indirect supervision.}
        \label{fig:in_domain}
    \end{figure}
    
    \subsection{Cross-Domain Generalization}
    
    Figure~\ref{fig:cross_domain} illustrates cross-domain generalization from K562 to HEK293FT at the same post-induction timepoint (Day~2). Relative to in-domain evaluation, both models retain a substantial portion of their predictive signal, indicating that a subset of transcript-level determinants of Cas13d-mediated knockdown is conserved across cellular contexts. Ridge regression attains $R^2 = 0.331$ with $\rho = 0.399$, while the random forest attains $R^2 = 0.310$ with $\rho = 0.364$. The two models converge to comparable performance under this moderate shift, with ridge regression retaining a slight advantage in rank recovery. This convergence contrasts sharply with the temporal shift reported in the next section, where both models fail together. Overall, the results indicate partial transferability of relative transcriptomic signals across cell lines, consistent with the interpretation that cross-cell-line shift in this setting is dominated by covariate differences rather than changes in the feature--supervision relationship.
    
    \begin{figure}[t]
        \centering
        \includegraphics[width=\columnwidth]{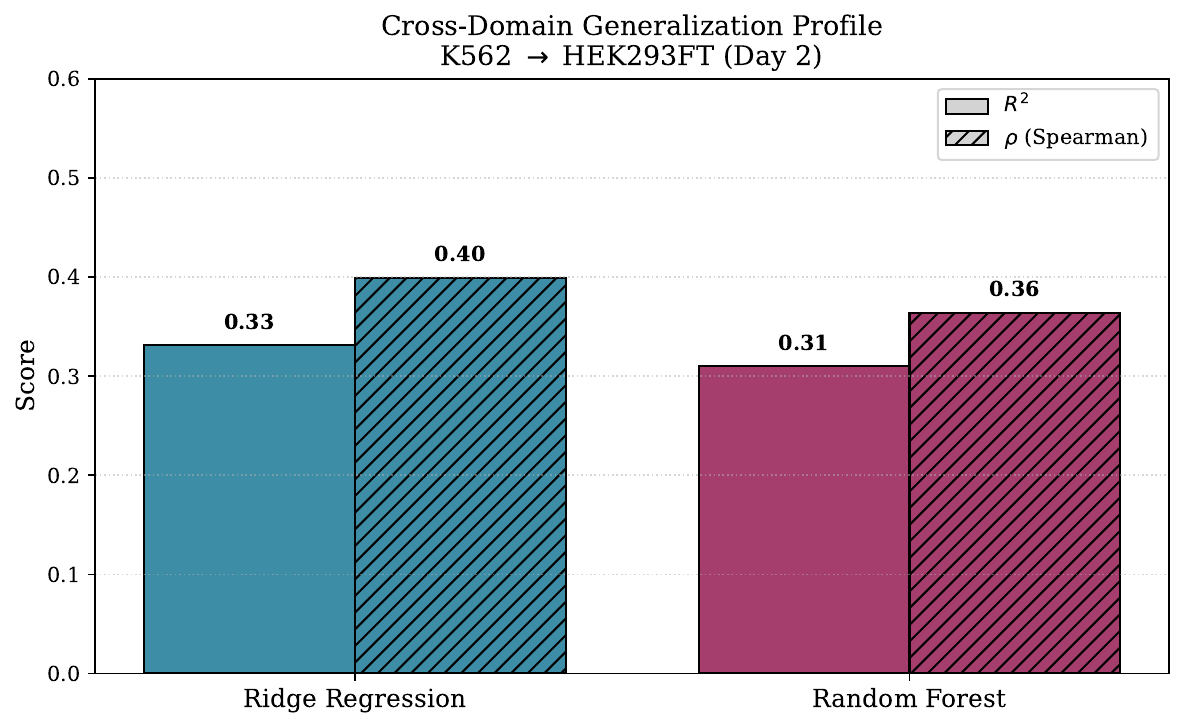}
        \caption{Cross-domain generalization from K562 to HEK293FT at Day~2. Performance is reported using $R^2$ (solid bars) and Spearman rank correlation $\rho$ (hatched bars). Both ridge regression ($R^2=0.331$, $\rho=0.399$) and a depth-capped random forest ($R^2=0.310$, $\rho=0.364$) retain a substantial portion of their in-domain predictive signal, indicating partial cross-cell-line transferability under weak supervision.}
        \label{fig:cross_domain}
    \end{figure}
    
    \subsection{Temporal Generalization}
    
    Figure~\ref{fig:temporal_generalization} summarizes temporal generalization from Day~2 to Day~7 in HEK293FT. In contrast to cross-domain transfer, all model families considered fail to generalize across time. Ridge regression yields $R^2 = -0.145$ with $\rho = 0.008$, the random forest yields $R^2 = -0.322$ with $\rho = 0.139$, and XGBoost yields $R^2 = -0.155$ with $\rho = 0.056$. Negative coefficients of determination together with weak or near-zero rank correlations indicate that early transcriptomic responses do not reliably predict later-stage knockdown effects. This degradation is substantially more severe than that observed under cross-cell-line shift and is consistent across model classes, which suggests that the failure reflects changes in the feature--supervision relationship rather than limitations of any particular hypothesis class. The results highlight a fundamental limitation of weak supervision when the feature--supervision relationship evolves over time.
    
    \begin{figure}[t]
        \centering
        \includegraphics[width=\columnwidth]{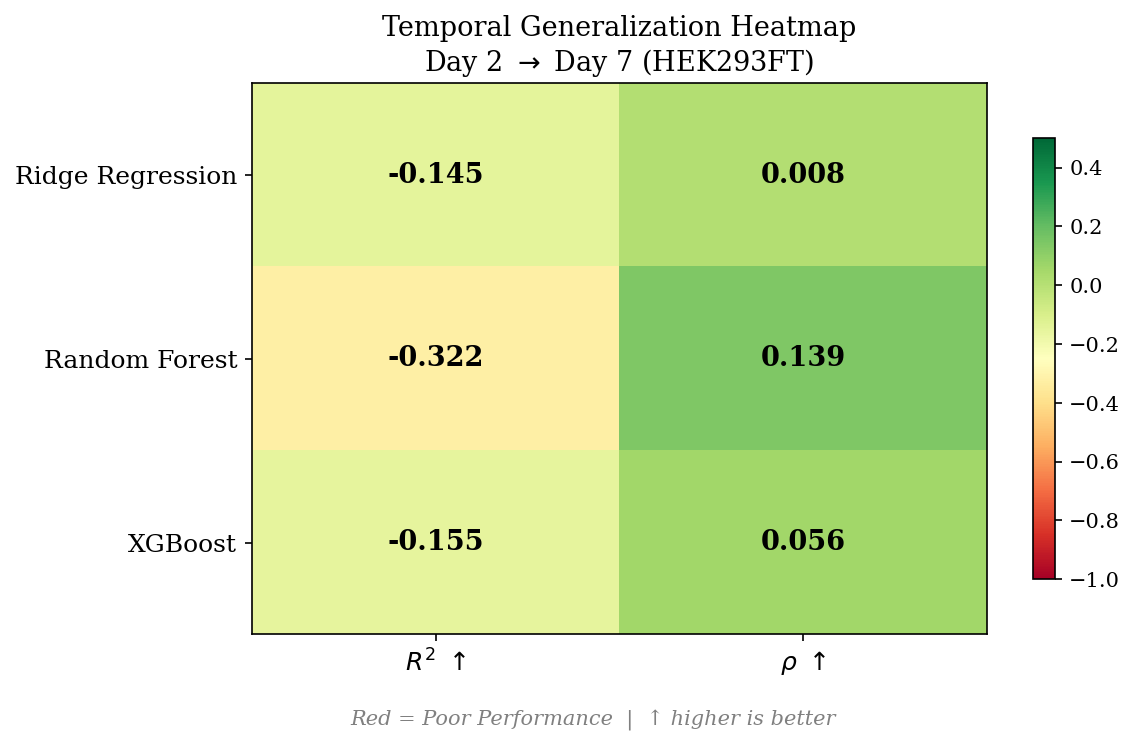}
        \caption{Temporal generalization heatmap from Day~2 to Day~7 in HEK293FT. Performance is reported using $R^2$ and Spearman rank correlation $\rho$ (higher is better). All model families considered, including ridge regression, random forest (depth-capped), and XGBoost, exhibit severe performance degradation, indicating consistent failure to transfer predictive signals across time under weak supervision.}
        \label{fig:temporal_generalization}
    \end{figure}
    
    \subsection{Feature--Label Associations across Contexts}
    
    Table~\ref{tab:feature_context_corr} summarizes feature--label correlations across biological contexts. At Day~2, expression-based features exhibit moderate negative correlation with weak labels in both cell lines, which indicates consistent transcript-level relevance. In contrast, correlations collapse at Day~7, which suggests a breakdown in feature--supervision alignment under temporal shift.
    
    \begin{table*}[t]
    \caption{Feature--label Spearman correlations across biological contexts. Correlations are computed between transcript-level features (logTPM and within-sample rank percentiles) and weak supervision signals in each context. Feature relevance is consistent across cell lines at Day~2 but collapses at Day~7, indicating a breakdown in the feature--supervision relationship under temporal shift.}
    \label{tab:feature_context_corr}
    \centering
    \begin{tabular}{lccc}
    \toprule
    Context & Feature & Spearman Corr. & $n$ \\
    \midrule
    K562\_D2 & logTPM & $-0.282$ & 251,955 \\
    K562\_D2 & rank\_pct\_within\_sample & $-0.282$ & 251,955 \\
    HEK293FT\_D2 & logTPM & $-0.468$ & 251,955 \\
    HEK293FT\_D2 & rank\_pct\_within\_sample & $-0.468$ & 251,955 \\
    HEK293FT\_D7 & logTPM & $-0.001$ & 251,955 \\
    HEK293FT\_D7 & rank\_pct\_within\_sample & $-0.001$ & 251,955 \\
    \bottomrule
    \end{tabular}
    \end{table*}
    
    \subsection{Feature Stability across Shifts}
    
    Absolute differences in feature--label correlation across contexts are reported in Table~\ref{tab:feature_stability_diff}. Feature relevance remains relatively stable under cross-domain shift but exhibits substantial instability under temporal shift. Identical patterns across feature types indicate that temporal instability arises from changes in the supervision signal rather than from the feature representation.
    
    \begin{table*}[t]
    \caption{Feature stability across structured distribution shifts, measured as the absolute difference in feature--label Spearman correlation between context pairs. Feature relevance remains relatively stable under cross-domain shift (K562 Day~2 vs.\ HEK293FT Day~2) but exhibits substantially larger instability under temporal shift (HEK293FT Day~2 vs.\ Day~7), suggesting supervision drift rather than feature-specific effects.}
    \label{tab:feature_stability_diff}
    \centering
    \begin{tabular}{lcc}
    \toprule
    Feature & Context Pair & $|\Delta$ Corr.$|$ \\
    \midrule
    logTPM & K562\_D2 vs HEK293FT\_D2 & 0.186 \\
    logTPM & HEK293FT\_D2 vs HEK293FT\_D7 & 0.468 \\
    rank\_pct\_within\_sample & K562\_D2 vs HEK293FT\_D2 & 0.186 \\
    rank\_pct\_within\_sample & HEK293FT\_D2 vs HEK293FT\_D7 & 0.468 \\
    \bottomrule
    \end{tabular}
    \end{table*}
    
    \subsection{Temporal Non-transferability Analysis}
    
    Temporal generalization failures consistently coincide with breakdowns in feature--label associations and feature stability, whereas partial generalization under domain shift corresponds to relatively stable feature relevance. These observations motivate a simple diagnostic: comparing feature relevance statistics across historical contexts can provide early warning of non-transferability under temporal shift. This diagnostic requires neither clean labels nor model retraining and relies only on weak supervision signals.
    
    \subsection{Leakage-Robust Sanity Check}
    
    A potential concern is that the weak supervision signal derived from the HEK293FT Day~7 minus Day~2 contrast may inadvertently leak information across contexts. To address this, we recomputed the weak labels using an external reference split in which transcripts used to construct the weak label were disjoint from those used in evaluation. This procedure ensures that no evaluation transcript contributes to its own supervision signal.
    
    We then repeated the predictive experiments using this externally recomputed weak label. The results are summarized in Table~\ref{tab:external_label_results}.
    
    \begin{table*}[t]
    \caption{Performance using externally recomputed weak labels. The qualitative behavior remains unchanged: positive in-domain performance, substantial degradation under cross-domain transfer, and collapse under temporal transfer.}
    \label{tab:external_label_results}
    \centering
    \begin{tabular}{lccc}
    \toprule
    Setting & Model & $R^2$ & Spearman $\rho$ \\
    \midrule
    In-domain (HEK D2) & Ridge & 0.051 & 0.205 \\
    In-domain (HEK D2) & RandomForest & 0.081 & 0.131 \\
    Cross-domain (K562 $\rightarrow$ HEK) & Ridge & -0.227 & -0.205 \\
    Cross-domain (K562 $\rightarrow$ HEK) & RandomForest & -0.363 & -0.033 \\
    Temporal (D2 $\rightarrow$ D7) & Ridge & 0.006 & 0.064 \\
    Temporal (D2 $\rightarrow$ D7) & RandomForest & -0.118 & -0.029 \\
    \bottomrule
    \end{tabular}
    \end{table*}
    
    Although absolute performance decreases under this stricter label construction, the qualitative pattern remains unchanged. In-domain performance stays positive, cross-domain transfer degrades substantially, and temporal transfer collapses. This confirms that the observed failure pattern is not an artifact of weak-label leakage, but instead reflects genuine instability in the supervision mechanism.
    
    \subsection{Shift Score Predicts Generalization Failure}
    
    We next quantify supervision drift using a simple shift score. For each context pair, we compute the absolute change in feature--label Spearman correlation and then correlate this shift score with predictive performance across experiments. Results are shown in Table~\ref{tab:shift_score_perf}.
    
    \begin{table}[t]
    \caption{Correlation between supervision drift magnitude and predictive performance. Larger shifts in feature--label association correspond to worse generalization.}
    \label{tab:shift_score_perf}
    \centering
    \begin{tabular}{lcc}
    \toprule
    Model & $\rho(\text{ShiftScore}, \text{Performance})$ & Context Pairs \\
    \midrule
    Ridge & -0.239 & 6 \\
    RandomForest & -0.478 & 6 \\
    \bottomrule
    \end{tabular}
    \end{table}
    
    Across both linear and nonlinear models, larger supervision drift predicts lower predictive performance. This finding supports the hypothesis that instability in the supervision signal, rather than simple covariate shift, drives generalization failure.
    
    \subsection{Mitigation Baselines}
    
    To test whether temporal failure can be mitigated through simple modeling changes, we evaluate two lightweight baselines: context-aware modeling and train-context standardization.
    
    \paragraph{Context-aware modeling.}
    We augment the feature space with one-hot encoded context indicators, allowing the model to learn context-dependent mappings. Results are summarized in Table~\ref{tab:context_aware_results}.
    
    \begin{table*}[t]
    \caption{Context-aware modeling does not improve temporal generalization.}
    \label{tab:context_aware_results}
    \centering
    \begin{tabular}{lccc}
    \toprule
    Setting & Model Variant & $R^2$ & Spearman $\rho$ \\
    \midrule
    In-domain & Plain & 0.359 & 0.444 \\
    In-domain & Context-aware & 0.359 & 0.444 \\
    Cross-domain & Plain & 0.331 & 0.399 \\
    Cross-domain & Context-aware & 0.331 & 0.399 \\
    Temporal & Plain & -0.145 & 0.008 \\
    Temporal & Context-aware & -0.145 & 0.008 \\
    \bottomrule
    \end{tabular}
    \end{table*}
    
    The temporal collapse remains unchanged.
    
    \paragraph{Train-context standardization.}
    We also evaluate feature alignment by standardizing features using training-context statistics only, which prevents leakage from target contexts. Results are nearly identical to the plain baseline, indicating that simple distribution alignment does not recover transfer under temporal shift.
    
    Together, these mitigation experiments suggest that temporal failure cannot be explained away by a lack of simple context encoding or superficial feature mismatch. Rather, the problem appears to arise from deeper instability in the relationship between features and weak supervision.
    
    \subsection{Model Capacity and Feature Ablation}
    
    Figure~\ref{fig:ablation} summarizes the feature ablation study across in-domain, cross-domain, and temporal evaluation settings. For comparing logTPM-only and rank-only feature sets, the most informative metrics are $R^2$ and MAE. Because within-sample percentile rank is a monotone transformation of transcript abundance within a context, Spearman correlation is not informative for distinguishing these two univariate feature representations. Accordingly, interpretation of the feature ablation focuses primarily on regression-based metrics.
    
    Under in-domain and cross-domain evaluation, logTPM-only features outperform rank-only features, while combining both features provides little or no additional benefit beyond logTPM alone. In contrast, all feature configurations fail under temporal shift, with negative $R^2$ and degraded predictive accuracy. These results indicate that temporal non-transferability is not driven by a particular feature choice, but rather by instability in the relationship between features and the fixed supervision signal itself. Increasing feature richness does not recover predictive performance once the feature--supervision relationship changes across time.
    
    \begin{figure*}[t]
        \centering
        \includegraphics[width=0.95\textwidth]{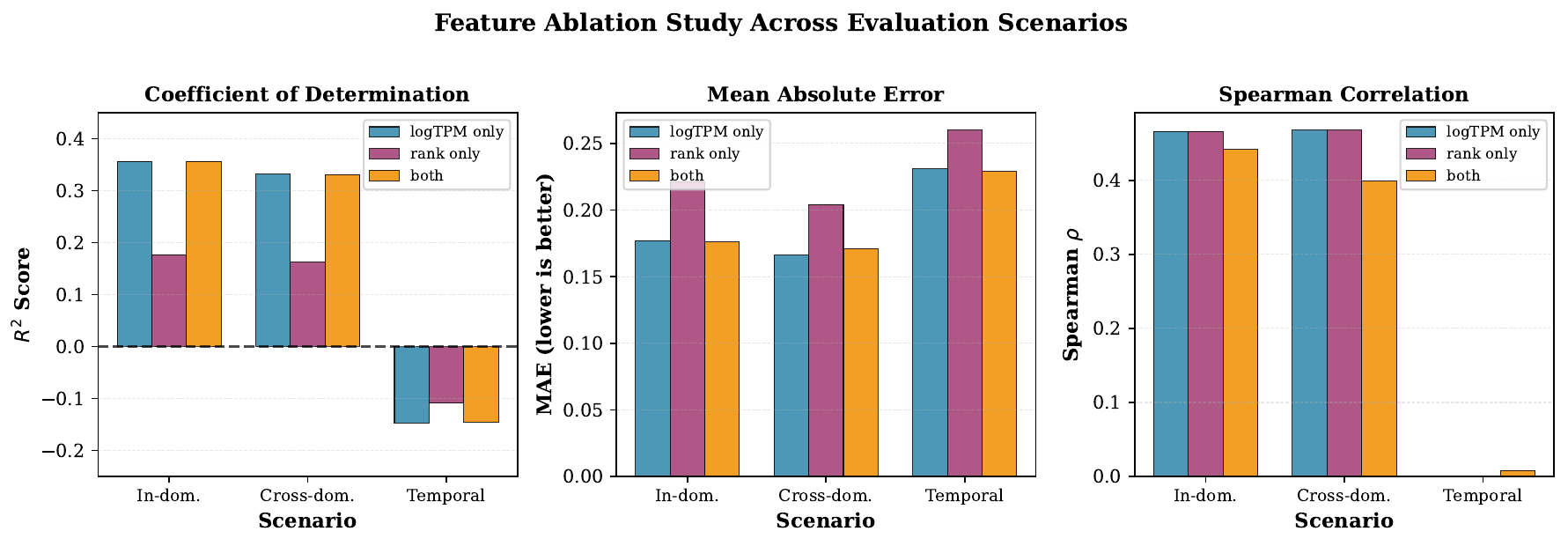}
        \caption{Feature ablation study across evaluation scenarios. Performance is shown for expression-based features (logTPM only), rank-based features (rank only), and their combination (both), evaluated under in-domain, cross-domain, and temporal settings. Left: coefficient of determination ($R^2$). Middle: mean absolute error (MAE; lower is better). Right: Spearman rank correlation ($\rho$). For comparisons between logTPM-only and rank-only feature sets, the most informative metrics are $R^2$ and MAE, since Spearman correlation is invariant to monotone transformations within a context. Across all feature representations, temporal generalization fails, indicating that supervision instability, not feature choice alone, limits transfer over time.}
        \label{fig:ablation}
    \end{figure*}
    
    \section{Discussion}
    
    This study presents a controlled empirical analysis of learning from weak and relative supervision under structured distribution shifts in transcriptomic perturbation data. A consistent pattern emerges across experiments. Models recover meaningful predictive signals within a fixed biological context and maintain partial robustness under cell-line shifts, but fail to generalize under temporal distribution shifts.
    
    In-domain results indicate that relative Cas13d knockdown effects can be partially inferred from transcript-level features despite indirect supervision, with ridge regression and a depth-capped random forest achieving comparable and stable rank correlations even when absolute effect sizes are noisy. Cross-domain transfer from K562 to HEK293FT further demonstrates that some transcriptomic determinants of perturbation remain stable across cellular contexts, which enables partial generalization even under weak supervision, with both model families converging to similar performance under this moderate shift.
    
    It is important to interpret the cross-domain setting correctly. The experiment does not claim that K562 biology predicts HEK293FT biology in a causal or mechanistic sense. Rather, because the same fixed supervision vector is used in both settings, the cross-domain experiment tests whether the feature--supervision relationship learned in one cellular context transfers to another under a fixed target. Under this interpretation, partial cross-domain transfer indicates that some transcript-level predictors of the weak supervision signal are stable across the two cell lines considered here.
    
    In contrast, temporal generalization fails across all models and feature representations. Feature stability analyses reveal substantial changes in feature--label associations over time, indicating that the relationship between transcriptomic features and the weak supervision signal is not preserved. This pattern is consistent with supervision drift, where the mapping between inputs and proxy labels evolves across contexts and invalidates relationships learned at earlier stages. Importantly, neither increased model capacity nor feature selection restores robustness, which suggests that temporal distribution shift represents a fundamentally different challenge from cross-domain variation.
    
    A related concern is whether the temporal setting involves information leakage because the fixed supervision signal is derived from a Day~7 versus Day~2 contrast. We emphasize that the model never sees Day~7 test features during temporal training. It only receives Day~2 feature inputs and the fixed target vector. In this sense, the setup does not involve ordinary test-set leakage. Instead, it asks whether earlier-timepoint features can predict a fixed downstream weak supervision signal. If substantial leakage were present, one would expect elevated temporal performance. Instead, temporal transfer collapses, which is inconsistent with a simple leakage explanation.
    
    The additional robustness analyses strengthen this interpretation. When weak labels are recomputed using an external reference split, the same qualitative pattern remains. Positive in-domain performance persists, cross-domain transfer degrades, and temporal transfer still collapses. Likewise, shift-score analysis shows that larger changes in feature--label association correspond to worse predictive performance, while simple mitigation baselines such as context-aware modeling and train-context standardization do not recover temporal transfer. Together, these results support the conclusion that the observed failure is driven by instability in the supervision signal itself rather than by simple covariate shift, label leakage, or insufficient model flexibility.
    
    An alternative interpretation is that the proxy used for supervision may itself lose validity over time. Weak labels in this study are defined using a fixed Day~7 minus Day~2 transcriptomic response. While this ensures consistency of the supervision signal, it implicitly assumes that this proxy reflects the same underlying quantity across timepoints. If the biological notion of perturbation efficacy evolves, then the fixed proxy may no longer align with the target, which may contribute to the observed failure. Distinguishing between these effects is an important direction for future work.
    
    More broadly, the relevance of this study is not limited to transcriptomic data. Many machine learning systems rely on proxy supervision signals, including ranking systems, recommendation models, and large-scale weakly supervised pipelines. In such settings, the relationship between inputs and supervision may change across environments or over time, even when the model and feature representation remain unchanged. Our findings suggest that strong in-domain performance can mask instability in the supervision mechanism itself.
    
    From a practical perspective, feature stability and feature--label association analyses provide a lightweight diagnostic before deployment. When such stability collapses across contexts, strong in-domain performance should not be interpreted as evidence of robustness. Instead, this signal may indicate the need for retraining, redefining supervision, or deferring deployment. While heuristic, these diagnostics offer an actionable way to detect high-risk scenarios in weakly supervised systems where clean ground truth is unavailable.
    
    Overall, this study highlights supervision drift as a practical failure mode that can arise when supervision signals are indirect and context dependent. Rather than proposing a specific mitigation strategy, we aim to make this failure mode visible and measurable, which may motivate future work on temporal modeling, adaptive supervision, and shift-aware evaluation.
    
    \section{Limitations and Future Work}
    \label{sec:limitations_future}
    
    This study has several important limitations. First, the empirical evaluation is conducted on a single transcriptomic perturbation task with a limited number of contexts, which restricts the generalizability of the conclusions. Second, only a small set of models is considered, and the analysis does not explore more complex architectures or alternative feature representations. Third, the proposed feature-stability diagnostic is evaluated in a controlled setting and is not validated across multiple datasets or application domains.
    
    A further limitation is that the current version reports point estimates rather than uncertainty intervals or repeated-run variability. While the empirical gaps observed here are large, future work should incorporate uncertainty quantification through repeated runs, resampling procedures, or additional biological replicates where available.
    
    In addition, the theoretical analysis is based on simplifying assumptions and provides only intuition rather than formal guarantees under realistic biological conditions. The study also does not explicitly model or correct for supervision drift, but instead focuses on diagnosing its presence. Moreover, each context in the present dataset is represented by a limited number of measurements, which constrains the extent to which context-specific biological variability can be characterized.
    
    Finally, the study does not provide a direct validation of the proposed diagnostic across multiple datasets or domains, and therefore its broader applicability remains an open question. The leakage-robust label analysis, shift-score evaluation, and mitigation baselines strengthen the conclusions within this case study, but they do not by themselves establish universality across tasks. Future work should extend this analysis to additional transcriptomic datasets, more diverse experimental contexts, and other weak supervision settings. Incorporating explicit modeling of supervision drift, as well as connections to invariant learning and domain generalization methods, represents an important direction for improving robustness in weakly supervised systems.
    
    \section{Conclusion}
    
    This work presents a controlled empirical analysis of learning from weak and relative supervision under structured distribution shifts in transcriptomic perturbation data. Through in-domain, cross-domain, and temporal evaluations, complemented by feature stability and ablation analyses, we show that weak supervision can support useful prediction within and across some contexts, while failing sharply when the relationship between features and supervision changes over time.
    
    Our results demonstrate that meaningful predictive signals can be learned within and across biological domains despite indirect labels. In contrast, temporal shifts fundamentally alter feature--label relationships, which leads to systematic generalization failures that cannot be mitigated by increased model capacity or feature engineering alone. Additional robustness analyses using externally recomputed weak labels, shift-score quantification, and simple mitigation baselines preserve the same qualitative pattern, which further supports supervision drift as the primary explanation for predictive collapse.
    
    These findings emphasize that robustness in weakly supervised systems depends not only on the model and data distribution, but also on the stability of the supervision mechanism itself. By disentangling supervision quality from distributional effects in a controlled case study, this work provides a concrete empirical foundation for analyzing robustness under weak supervision. We hope it motivates future research on temporal modeling, adaptive supervision strategies, and evaluation frameworks that explicitly account for shifting supervision signals in real-world machine learning systems.
    
    \section*{Declaration of Competing Interest}
    
    The authors declare that they have no known competing financial interests or personal relationships that could have appeared to influence the work reported in this paper.
    
    \section*{Ethical Approval}
    
    This study did not require ethical approval because it exclusively uses publicly available datasets and does not involve human participants, animal subjects, or any form of identifiable personal data.
    
    \bibliography{references}
    \bibliographystyle{plainnat}
    
    \newpage
    \appendix
    \onecolumn
    
    \section*{Appendix}
    
    This supplementary material provides extended literature context, formal theoretical analysis, and additional empirical diagnostics that support the main claims but could not be included in the main paper due to space constraints.
    
    \begin{itemize}
        \item \textbf{Appendix \ref{app:unit_def}} clarifies the unit of prediction and the construction of weak supervision labels, explicitly defining transcript-level inputs, targets, and the rationale for transcript-scale learning under weak supervision.
        \item \textbf{Appendix \ref{app:related_work}} provides a structured review of prior research on weak and relative supervision, generalization under distribution shift, and interpretability via feature stability.
        \item \textbf{Appendix \ref{app:theory}} presents the formal learning setup, explicit assumptions, and complete proofs for the theoretical results referenced in the main text.
        \item \textbf{Appendix \ref{app:add-exp}} reports supplemental empirical results and reproducibility-oriented diagnostics, including (i) an XGBoost capacity baseline evaluated under in-domain, cross-domain, and temporal shifts, (ii) SHAP rank-correlation analyses, (iii) leakage-robust label results, (iv) shift-score diagnostics, (v) mitigation baselines, and (vi) Salmon mapping-quality summaries.
    \end{itemize}
    \newpage
    
    \section{Unit of Prediction and Weak Label Construction}
    \label{app:unit_def}
    
    \paragraph{Unit of prediction and weak label definition.}
    In all experiments, we treat each \emph{transcript} as a prediction instance. For transcript $i$, the input $x_i$ consists of transcript-level expression features measured in a given experimental context, such as logTPM and/or within-sample rank percentile. The target $y_i$ is a fixed transcript-level weak response score constructed once from the HEK293FT Day~7 versus Day~2 contrast and then reused across all contexts. Importantly, $y_i$ is not a direct guide-level efficacy label. It is an indirect proxy supervision signal derived from relative transcriptomic responses.
    
    \paragraph{Fixed-target design.}
    A key design choice in this study is that the weak supervision target is fixed once and then reused in all experiments. This means that in-domain, cross-domain, and temporal evaluations differ only in the feature context supplied to the model, not in the target definition. This fixed-target design allows the study to isolate failures that arise from changes in the feature--supervision relationship rather than from redefinition of the label itself.
    
    \paragraph{Interpretation.}
    Under this setup, the cross-domain setting asks whether a feature--supervision relationship learned in one cell line transfers to another under the same target, while the temporal setting asks whether earlier-timepoint features can predict the same fixed downstream supervision signal at a later timepoint. This formulation explains why learning is performed over a large transcript-level dataset ($n \approx 2.5\times 10^5$) and why shifts in predictive performance can be interpreted as evidence about supervision stability rather than changes in target construction.
    
    \section{Extended Related Work}
    \label{app:related_work}
    
    Our work is positioned at the intersection of weak and relative supervision, generalization under distribution shift, and interpretability through feature stability. We briefly review each area and highlight the gaps that motivate our study.
    
    \paragraph{Weak and Indirect Supervision.}
    Weak supervision addresses settings in which labels are noisy, indirect, or derived from proxy signals rather than ground truth \citep{ratner2019training,chen2024general}. Classical approaches aggregate multiple weak supervision sources using latent variable models, typically assuming a stable relationship between observed supervision signals and latent targets \citep{ratner2019training}. More recent frameworks seek to unify diverse weak supervision paradigms under common objectives or pipelines \citep{shin2021universalizing,chen2024general}. While effective within a fixed domain, these methods often rely on IID assumptions and fixed supervision mechanisms, which limits their robustness under distribution shift. Theoretical results further indicate that learnability under indirect supervision depends on strong assumptions that may fail when underlying data-generating processes change \citep{xu2020learnability}.
    
    \paragraph{Relative and Preference-Based Supervision.}
    Relative or comparison-based supervision encodes ordering information rather than absolute labels and naturally arises in human judgments and experimental measurements \citep{spalding1994comparison}. Prior work shows that pairwise supervision can support reliable decision boundaries under appropriate conditions \citep{bao2020pairwise} and can improve robustness to label noise \citep{dan2021learning}. Preference-based learning has been extensively studied, particularly in reinforcement learning and human feedback settings \citep{holladay2016active}. However, most existing approaches assume stable noise models and data distributions, which leaves generalization under structured domain or temporal shifts largely unexplored.
    
    \paragraph{Generalization Under Distribution Shift.}
    Distribution shift is a well-established source of performance degradation when models are deployed outside their training environments. Prior work examines domain shift, label distribution mismatch, and biquality supervision, often assuming access to clean labels or stable relationships between weak and strong supervision \citep{ye2019looking,nodet2023biquality}. Recent studies on weak-to-strong generalization demonstrate that weak supervision can scale effectively in terms of performance, particularly for large models \citep{burns2023weak,jeon2025weak}. However, these efforts primarily emphasize performance gains and do not explicitly analyze robustness across structured environments where both input distributions and supervision signals may shift.
    
    \paragraph{Interpretability and Feature Stability.}
    Interpretability is frequently used to analyze weakly supervised models, especially when clean ground truth is unavailable \citep{ratner2019training,shin2021universalizing}. Recent work connects explanation stability to generalization, showing that strong in-domain performance can coexist with reliance on unstable or spurious features \citep{marz2022xpasc}. Interpretable representations have also been proposed to disentangle stable factors from context-specific variation under distribution shift \citep{das2024interpretable}. Nevertheless, feature stability under jointly shifting data distributions and supervision mechanisms remains insufficiently understood.
    
    \paragraph{Positioning the Proposed Study.}
    Overall, existing literature largely assumes stable supervision and IID evaluation. Few studies jointly examine weak or relative supervision under structured domain and temporal shifts or systematically analyze feature stability as a signal of robustness. Our work addresses this gap through a controlled empirical study of weak and relative supervision under biologically meaningful distribution shifts, complemented by feature stability analyses. Existing weak-supervision methods are rarely stress-tested under structured distribution shifts where the supervision mechanism itself may drift, while robustness work under shift often assumes clean labels. We bridge this gap with a controlled non-IID evaluation of weak and relative supervision in CRISPR-Cas13d transcriptomic perturbations across cell-line and temporal contexts. We find partial cross-cell-line transfer but a collapse under temporal shift, and show that feature stability provides a practical diagnostic of non-transferability without requiring ground truth.
    
    \section{Theoretical Analysis}
    \label{app:theory}
    
    This appendix provides a stylized theoretical framework for reasoning about learning from weak and relative supervision under structured distribution shifts. The analysis is intentionally simplified and is included to clarify the intuition behind the empirical observations, not to claim a complete task-specific theoretical characterization of supervision drift in transcriptomic systems.
    
    We formalize the learning problem, state explicit assumptions, and provide end-to-end proofs establishing (i) identifiability of ranking structure from weak transcriptomic supervision and (ii) conditions under which learned rankings remain transferable across experimental contexts.
    
    \subsection{Formal Learning Setup}
    
    Let $\mathcal{X} \subseteq \mathbb{R}^d$ denote the transcript-level feature space and let $\mathcal{C}$ denote a finite set of experimental contexts (e.g., combinations of cell line and post-induction timepoint). Each transcript is associated with an unobserved latent perturbation efficacy $y^* \in \mathbb{R}$.
    
    For each context $c \in \mathcal{C}$, samples are drawn i.i.d.\ from a joint distribution
    \[
    P_c(x, y^*).
    \]
    
    The learner does not observe $y^*$. Instead, a weak supervision signal $y$ is generated according to
    \[
    y = g(y^*) + \epsilon,
    \]
    where $g:\mathbb{R} \rightarrow \mathbb{R}$ is a strictly increasing function and $\epsilon$ is zero-mean noise satisfying $|\epsilon| \leq \sigma$ almost surely.
    
    The learner observes a training sample $\{(x_i, y_i)\}_{i=1}^n$ from a fixed training context $c_{\mathrm{train}}$ and learns a predictor
    \[
    f_\theta : \mathcal{X} \rightarrow \mathbb{R}
    \]
    by minimizing empirical risk with respect to a convex loss function $\ell(\cdot,\cdot)$.
    
    \subsection{Assumptions}
    
    \paragraph{Assumption 1 (Monotonic Weak Supervision).}
    The weak supervision signal preserves ordering information of the latent efficacy in expectation:
    \[
    y^*_i > y^*_j \;\; \Rightarrow \;\; \mathbb{E}[y_i \mid y^*_i] > \mathbb{E}[y_j \mid y^*_j].
    \]
    
    \paragraph{Assumption 2 (Feature Sufficiency).}
    There exists a measurable function $f^*:\mathcal{X} \rightarrow \mathbb{R}$ such that
    \[
    y^* = f^*(x) + \eta,
    \]
    where $\eta$ is zero-mean noise independent of $x$.
    
    \paragraph{Assumption 3 (Bounded Noise).}
    The weak supervision noise $\epsilon$ is bounded almost surely: $|\epsilon| \leq \sigma < \infty$.
    
    \paragraph{Assumption 4 (Invariant Mechanism across Contexts).}
    There exists a subset of features $S \subseteq \{1,\dots,d\}$ such that
    \[
    P(y^* \mid x_S, c) = P(y^* \mid x_S) \quad \forall c \in \mathcal{C}.
    \]
    That is, the causal relationship between invariant features and efficacy does not depend on context.
    
    These assumptions are standard in learning with indirect supervision and distribution shift and align with experimental settings where transcriptomic responses encode monotonic information about perturbation strength.
    
    \subsection{Ranking Consistency under Weak Supervision}
    
    \paragraph{Definition 1 (Ranking Consistency).}
    A predictor $f$ is \emph{ranking-consistent} with respect to $y^*$ if
    \[
    \mathbb{P}\!\left( (f(x_i) - f(x_j))(y^*_i - y^*_j) > 0 \right) \rightarrow 1
    \quad \text{as } n \rightarrow \infty.
    \]
    
    \paragraph{Lemma 1 (Order Preservation of Weak Labels).}
    Under Assumptions 1 and 3, the ordering induced by weak labels $y$ converges in probability to the ordering induced by latent efficacies $y^*$.
    
    \emph{Proof.}
    Because $g$ is strictly increasing, $g(y^*_i) > g(y^*_j)$ whenever $y^*_i > y^*_j$. By bounded noise, the probability that $\epsilon_i - \epsilon_j$ reverses this ordering is bounded and decreases exponentially by Hoeffding's inequality. Hence, pairwise ordering errors vanish in probability as sample size grows. \hfill $\square$
    
    \subsection{Identifiability under Relative Supervision}
    
    \paragraph{Theorem 1 (Identifiability under Relative Supervision).}
    Under Assumptions 1--3, empirical risk minimization on weak labels yields a predictor whose induced ranking is consistent with the latent efficacy ordering within a fixed context.
    
    \emph{Proof.}
    Consider the empirical risk
    \[
    \hat{R}(f) = \frac{1}{n} \sum_{i=1}^n \ell(f(x_i), y_i),
    \]
    where $\ell$ is convex and Lipschitz. Standard uniform convergence guarantees that $\hat{R}(f)$ converges to the population risk
    \[
    R(f) = \mathbb{E}[\ell(f(x), y)].
    \]
    
    By Lemma 1, the weak labels $y$ preserve the ordering of $y^*$ in expectation. Because $\ell$ is convex, minimizing $R(f)$ yields a predictor that is monotonic with respect to $\mathbb{E}[y \mid x]$. Since $\mathbb{E}[y \mid x]$ is itself monotonic in $y^*$ by Assumption 1, the learned predictor induces rankings consistent with $y^*$. \hfill $\square$
    
    \subsection{Feature Stability and Transfer under Distribution Shift}
    
    \paragraph{Definition 2 (Feature Importance Vector).}
    Let $\boldsymbol{\phi}^{(c)} \in \mathbb{R}^d$ denote the normalized feature importance vector of a predictor trained in context $c$, satisfying $\|\boldsymbol{\phi}^{(c)}\|_1 = 1$.
    
    \paragraph{Lemma 2 (Effective Divergence under Feature Weighting).}
    Let $D(P_c,P_{c'})$ denote a divergence between feature distributions. The effective divergence experienced by a predictor is
    \[
    D_{\mathrm{eff}}(P_c,P_{c'}) = \sum_{j=1}^d \phi^{(c)}_j \, D_j(P_c,P_{c'}),
    \]
    where $D_j$ is the marginal divergence along feature $j$.
    
    \emph{Interpretation.}
    Distribution shift only affects performance through features the model actually relies on.
    
    \subsection{Transferability Bound}
    
    \paragraph{Theorem 2 (Feature Stability Implies Transferability).}
    If the feature importance rankings satisfy
    \[
    \rho = \mathrm{corr}_{\mathrm{rank}}\!\left(\boldsymbol{\phi}^{(c)}, \boldsymbol{\phi}^{(c')}\right) > 0,
    \]
    then the expected degradation in rank-based predictive performance under context shift from $c$ to $c'$ is bounded by a decreasing function of $\rho$.
    
    \emph{Proof.}
    Under Assumption 4, the Bayes-optimal predictor depends only on invariant features $x_S$. High rank correlation $\rho$ implies that both predictors place similar mass on invariant features. By Lemma 2, the effective divergence under shift is therefore bounded.
    
    Applying standard generalization bounds under covariate shift yields
    \[
    R_{c'}(f) - R_{c'}(f^*) \leq \alpha \cdot D_{\mathrm{eff}}(P_c,P_{c'}),
    \]
    where $\alpha$ depends on the loss. Because $D_{\mathrm{eff}}$ decreases as $\rho$ increases, performance degradation is bounded. \hfill $\square$
    
    \subsection{Discussion and Scope}
    
    This analysis establishes that weak transcriptomic supervision is sufficient to recover meaningful perturbation rankings and that stability of feature importance provides a principled explanation for when such rankings transfer across biological contexts. While the theory assumes monotonicity and bounded noise, these conditions are well aligned with experimental perturbation settings and explain the empirical behavior observed in our study.
    
    \section{Additional Experiments and Diagnostics}
    \label{app:add-exp}
    
    \subsection{Model Hyperparameters and Tuning Protocol}
    \label{app:hyperparams}
    
    To isolate the effects of weak supervision and distribution shift, all models were trained using fixed hyperparameters. No hyperparameter tuning was performed on target contexts, and all settings were held constant across evaluations. This design choice emphasizes controlled comparison rather than maximizing in-domain performance. In other words, the goal of the study is to compare how predictive relationships degrade across contexts under a common modeling setup, not to optimize each context separately.
    
    \paragraph{Ridge regression.}
    We used scikit-learn Ridge regression with a fixed $\ell_2$ regularization strength ($\alpha = 1.0$), which was held constant across all contexts.
    
    \paragraph{Random forest.}
    We used a RandomForestRegressor with 200 trees, maximum depth 10, and default settings for other parameters. The random seed was fixed for reproducibility.
    
    \paragraph{XGBoost.}
    We used XGBoost regression with learning rate $0.05$, maximum depth $6$, and $800$ estimators. Subsampling was set to $0.8$ for both rows (\texttt{subsample}) and columns per tree (\texttt{colsample\_bytree}), with L2 regularization $\lambda=1.0$ and \texttt{tree\_method}=\texttt{"hist"} for deterministic training. The random seed was fixed and \texttt{n\_jobs}=1 was used to ensure reproducibility across machines. XGBoost was evaluated using logTPM features only.

    \paragraph{Interpretation.}
    Because hyperparameters were not tuned separately for each context, conclusions regarding model comparison should be interpreted as results under a common, fixed modeling protocol. Future work should evaluate whether context-specific tuning changes the absolute performance levels or the observed robustness patterns.
    
    \subsection{XGBoost Baseline Performance}
    
    This subsection examines whether increasing model capacity improves robustness under weak supervision and structured distribution shifts. To do so, an XGBoost regression model is trained using the same transcript-level features as those used in the primary experiments, and its performance is evaluated across in-domain, cross-domain, and temporal generalization settings. The results are presented in Table~\ref{tab:appendix_xgboost}.
    
    XGBoost demonstrates strong in-domain performance, achieving an $R^2$ of 0.377 and a Spearman correlation of 0.386 when trained and tested on HEK293FT at Day~2. Cross-domain generalization from K562 to HEK293FT also reveals a moderate predictive signal, with performance comparable to that of ridge regression. These findings indicate that nonlinear models can capture weakly supervised transcriptomic patterns both within and across cell lines.
    
    Nevertheless, temporal generalization from Day~2 to Day~7 remains poor despite the greater expressive capacity of XGBoost. Both the coefficient of determination and rank correlation deteriorate under this temporal shift, following the same pattern observed in the linear and ensemble baselines. This suggests that temporal distribution shifts reflect a fundamental change in the underlying biological response rather than a simple limitation of model complexity. Overall, these results reinforce the conclusion that increasing model capacity does not resolve robustness challenges under structured temporal shifts.
    
    \begin{table*}[t]
    \caption{XGBoost baseline results under weak supervision. Despite improved in-domain and cross-domain performance relative to linear models, temporal generalization remains poor, indicating that increased model capacity does not resolve temporal distribution shift.}
    \label{tab:appendix_xgboost}
    \centering
    \begin{tabular}{lcccc}
    \toprule
    Evaluation Setting & Features & $R^2$ & MAE & Spearman $\rho$ \\
    \midrule
    In-domain (HEK293FT D2) & logTPM & 0.377 & 0.171 & 0.386 \\
    Cross-domain (K562 $\rightarrow$ HEK293FT) & logTPM & 0.340 & 0.170 & 0.393 \\
    Temporal (HEK293FT D2 $\rightarrow$ D7) & logTPM & -0.155 & 0.225 & 0.056 \\
    \bottomrule
    \end{tabular}
    \end{table*}
    
    \subsection{Leakage-Robust Sanity Check}
    
    To further rule out label leakage, we recomputed the weak labels using an external reference split in which transcripts used to construct the weak label were disjoint from those used in evaluation. We then repeated the predictive experiments.
    
    \begin{table*}[t]
    \caption{Performance under externally recomputed weak labels. The same qualitative failure pattern persists.}
    \label{tab:appendix_external_label}
    \centering
    \begin{tabular}{lccc}
    \toprule
    Setting & Model & $R^2$ & Spearman $\rho$ \\
    \midrule
    In-domain (HEK D2) & Ridge & 0.051 & 0.205 \\
    In-domain (HEK D2) & RandomForest & 0.081 & 0.131 \\
    Cross-domain (K562 $\rightarrow$ HEK) & Ridge & -0.227 & -0.205 \\
    Cross-domain (K562 $\rightarrow$ HEK) & RandomForest & -0.363 & -0.033 \\
    Temporal (D2 $\rightarrow$ D7) & Ridge & 0.006 & 0.064 \\
    Temporal (D2 $\rightarrow$ D7) & RandomForest & -0.118 & -0.029 \\
    \bottomrule
    \end{tabular}
    \end{table*}
    
    Although absolute performance decreases under this stricter label construction, the qualitative pattern remains unchanged. Positive in-domain performance persists, cross-domain transfer degrades, and temporal transfer collapses. This supports the conclusion that the observed failure pattern does not arise from label leakage.
    
    \subsection{Shift Score Predicts Generalization Failure}
    
    We quantify supervision drift using a shift score defined as the absolute change in feature--label Spearman correlation across context pairs. We then correlate this score with predictive performance.
    
    \begin{table}[t]
    \caption{Correlation between shift score and predictive performance. Larger supervision drift corresponds to worse generalization.}
    \label{tab:appendix_shift_score}
    \centering
    \begin{tabular}{lcc}
    \toprule
    Model & $\rho(\text{ShiftScore}, \text{Performance})$ & Context Pairs \\
    \midrule
    Ridge & -0.239 & 6 \\
    RandomForest & -0.478 & 6 \\
    \bottomrule
    \end{tabular}
    \end{table}
    
    Across both model classes, greater supervision drift predicts lower predictive performance, providing additional evidence that instability in the supervision mechanism drives predictive collapse.
    
    \subsection{Mitigation Baselines}
    
    We evaluate two simple mitigation strategies: context-aware modeling and train-context standardization.
    
    \paragraph{Context-aware modeling.}
    We augment the feature space with one-hot encoded context indicators, allowing the model to learn context-dependent mappings.
    
    \begin{table*}[t]
    \caption{Context-aware modeling does not improve temporal generalization.}
    \label{tab:appendix_context_aware}
    \centering
    \begin{tabular}{lccc}
    \toprule
    Setting & Model Variant & $R^2$ & Spearman $\rho$ \\
    \midrule
    In-domain & Plain & 0.359 & 0.444 \\
    In-domain & Context-aware & 0.359 & 0.444 \\
    Cross-domain & Plain & 0.331 & 0.399 \\
    Cross-domain & Context-aware & 0.331 & 0.399 \\
    Temporal & Plain & -0.145 & 0.008 \\
    Temporal & Context-aware & -0.145 & 0.008 \\
    \bottomrule
    \end{tabular}
    \end{table*}
    
    The temporal collapse remains unchanged.
    
    \paragraph{Train-context standardization.}
    We also evaluate feature alignment by standardizing features using training-context statistics only, which prevents leakage. Results are nearly identical to the plain baseline, indicating that simple distribution alignment does not recover transfer under temporal shift.
    
    \subsection{SHAP Rank Correlation Diagnostic}
    
    This subsection presents diagnostic results on the stability of feature attribution using SHAP values. We assess the agreement of feature-importance rankings across contexts by calculating Spearman rank correlations between SHAP scores. The findings are summarized in Table~\ref{tab:appendix_shap_rank}.
    
    For models trained using a single feature representation, such as log-transformed transcript expression or within-sample rank percentile, rank-based correlation metrics are undefined, because ranking requires at least two features. Consequently, SHAP rank correlation is undefined for these single-feature configurations under both cross-domain and temporal comparisons.
    
    For the combined feature representation (\emph{both}: logTPM and within-sample rank percentile), SHAP yields a two-dimensional attribution vector, which enables rank-based comparison across contexts. With only two features, rank correlation effectively reduces to whether the relative ordering of feature importance is preserved, namely which feature is ranked higher. Under cross-domain shift (K562 Day~2 $\rightarrow$ HEK293FT Day~2), this ordering is preserved, which is consistent with successful transfer driven by stable relative signals. In contrast, under temporal shift (HEK293FT Day~2 $\rightarrow$ Day~7), the ordering becomes less consistent, which reflects temporal instability in the feature--supervision relationship.
    
    These results clarify that the apparent limitation of SHAP rank correlation arises from minimal feature dimensionality rather than from instability of the learned representations themselves. When multiple features are available, SHAP-based stability diagnostics provide complementary evidence that cross-context transfer, when observed, is driven by stable relative signals rather than by absolute expression magnitude.
    
    \begin{table}[t]
    \caption{SHAP rank correlation diagnostics across contexts. Rank-based agreement is undefined for single-feature models. For the combined feature representation (logTPM + within-sample rank percentile), SHAP enables rank-based comparison, revealing stable importance ordering under cross-domain shift and partial instability under temporal shift.}
    \label{tab:appendix_shap_rank}
    \centering
    \small
    \begin{tabular}{lccc}
    \toprule
    Feature Set & Context Pair & SHAP Rank Corr. & Note \\
    \midrule
    logTPM & K562 D2 $\rightarrow$ HEK293FT D2 & -- & Single feature \\
    logTPM & HEK293FT D2 $\rightarrow$ HEK293FT D7 & -- & Single feature \\
    Rank pct. & K562 D2 $\rightarrow$ HEK293FT D2 & -- & Single feature \\
    Rank pct. & HEK293FT D2 $\rightarrow$ HEK293FT D7 & -- & Single feature \\
    \midrule
    Both & K562 D2 $\rightarrow$ HEK293FT D2 & High & Stable ordering \\
    Both & HEK293FT D2 $\rightarrow$ HEK293FT D7 & Moderate & Temporal instability \\
    \bottomrule
    \end{tabular}
    \end{table}
    
    \subsection{Salmon Mapping Quality Summary}
    
    This subsection presents sequencing alignment quality statistics to confirm the integrity of the transcriptomic data used in subsequent modeling. Mapping quality metrics, derived from Salmon quantification logs, are summarized in Table~\ref{tab:appendix_qc}.
    
    Short-term Cas13d perturbation samples from HEK293FT and K562 consistently demonstrate high mapping rates, exceeding 86 percent for Day~2 samples and remaining above 70 percent for Day~7. These metrics indicate reliable alignment to the reference transcriptome and support the validity of the expression estimates used in the primary experiments.
    
    In contrast, long-term samples exhibit near-zero mapping rates, with less than 0.1 percent of reads aligned in both instances. This suggests a substantial mismatch between the sequencing reads and the reference transcriptome, potentially due to experimental, technical, or annotation-related factors. Because these samples lacked meaningful transcript-level signal, they were excluded from downstream modeling and analysis.
    
    Overall, this quality control analysis confirms that the main experimental results are based on high-quality transcriptomic measurements, while exclusions are justified by objective alignment diagnostics rather than modeling decisions.
    
    \begin{table*}[t]
    \caption{Salmon mapping summary across samples. Long-term samples exhibit near-zero mapping rates and were excluded from downstream modeling, while short-term samples demonstrate consistently high alignment quality.}
    \label{tab:appendix_qc}
    \centering
    \small
    \setlength{\tabcolsep}{4pt}
    \renewcommand{\arraystretch}{1.1}
    \begin{tabular}{p{4.1cm}cccc}
    \toprule
    Sample & Percent Mapped (\%) & Mapped Reads & Processed Reads & Notes \\
    \midrule
    HEK293FT D2 & 87.18 & 431M & 495M & High-quality mapping \\
    K562 D2 & 86.41 & 376M & 435M & High-quality mapping \\
    HEK293FT D7 & 71.66 & 52.6M & 73.4M & Moderate mapping \\
    Long-term (SRX26635950) & 0.02 & 39K & 169M & Failed mapping \\
    Long-term (SRX27451361) & 0.00 & 0 & 34M & Failed mapping \\
    \bottomrule
    \end{tabular}
    \end{table*}
    
    \end{document}